\documentclass{article} 

\usepackage{graphics,todonotes,paralist}
\usepackage{algorithm,algorithmic} 
\usepackage{hyperref}       
\usepackage{amsmath,amssymb,amsthm}
\usepackage{natbib}
\usepackage{fullpage}
\usepackage{multirow}
\usepackage{xspace}
\newtheorem{definition}{Definition}
\newtheorem{theorem}{Theorem}

\newcommand{\N}{\mathbb{N}}  
\newcommand{\x}{\mathbf{x}}  
\newcommand{\F}{0}  
\newcommand{\T}{1}  

\newcommand{\Regret}{\text{Regret}}
\newcommand{\poly}{\mathop{\text{poly}}}

\newcommand{\indicator}{\mathbf{1}}

\newcommand{\A}{{\mathcal A}}
\newcommand{\D}{{\mathcal D}}

\def\algdisj{\textsf{Alg}_\textsf{disj}}
\def\algosco{\textsf{Alg}_\textsf{osco}}

\newcommand{\ocombo}{online combinatorial optimization}

\newcommand{\oscombo}{online sleeping combinatorial optimization\xspace}

\newcommand{\Us}{U_s}
\newcommand{\prob}[1]{\textsc{#1}}

\title{Hardness of Online Sleeping Combinatorial Optimization Problems}

\author{Satyen Kale\thanks{Current affiliation: Google Research.}\ \ $^\dagger$ \\
 	Yahoo Research \\
 	\texttt{satyen@satyenkale.com}
\and
  Chansoo Lee\thanks{This work was done while the authors were at Yahoo Research.} \\
  Univ. of Michigan, Ann Arbor \\
  \texttt{chansool@umich.edu}
\and
  D\'avid P\'al \\
  Yahoo Research \\
  \texttt{dpal@yahoo-inc.com}
}

\begin{document}

\date{}
\maketitle

\begin{abstract}
We show that several online combinatorial optimization problems that admit
efficient no-regret algorithms become computationally hard in the sleeping
setting where a subset of actions becomes unavailable in each round.
Specifically, we show that the sleeping versions of these problems are at least
as hard as PAC learning DNF expressions, a long standing open problem. We show
hardness for the sleeping versions of \prob{Online Shortest Paths}, \prob{Online Minimum Spanning Tree}, \prob{Online $k$-Subsets}, \prob{Online $k$-Truncated Permutations}, \prob{Online Minimum Cut}, and \prob{Online Bipartite Matching}. The hardness result for the sleeping version of the Online Shortest Paths problem resolves an open problem presented at COLT~2015~\citep{Koolen-Warmuth-Adamskiy-2015}.
\end{abstract}



\section{Introduction}

Online learning is a sequential decision-making problem where learner
repeatedly chooses an action in response to adversarially chosen losses for the
available actions. The goal of the learner is to minimize the \emph{regret},
defined as the difference between the total loss of the algorithm and the loss
of the best fixed action in hindsight. In online combinatorial optimization,
the actions are subsets of a ground set of \emph{elements} (also called
\emph{components}) with some combinatorial structure. The loss of an action is
the sum of the losses of its elements. A particular well-studied instance is
the \prob{Online Shortest Path} problem \citep{Takimoto-Warmuth-2003} on a
graph, in which the actions are the paths between two fixed vertices and the
elements are the edges.

We study a \emph{sleeping} variant of online combinatorial optimization where
the adversary not only chooses losses but \emph{availability} of the
elements every round. The unavailable elements are called \emph{sleeping} or
\emph{sabotaged}. In \prob{Online Sabotaged Shortest Path} problem, for example, 
the adversary specifies unavailable edges every round, and
consequently the learner cannot choose any path using those edges.
A straightforward application of the sleeping experts algorithm proposed by \citet{Freund-Schapire-Singer-Warmuth-1997} gives a no-regret learner, but it takes exponential time (in the input graph size) every round. 
The design of a computationally efficient no-regret algorithm for \prob{Online Sabotaged Shortest Path} problem was presented as an open problem at COLT 2015 by \citet{Koolen-Warmuth-Adamskiy-2015}.

In this paper, we resolve this open problem and prove that \prob{Online Sabotaged
Shortest Path} problem is computationally hard. Specifically, we show that a
polynomial-time low-regret algorithm for this problem implies a polynomial-time
algorithm for PAC learning DNF expressions, which is a long-standing open
problem. The best known algorithm for PAC learning DNF expressions on $n$
variables has time complexity $2^{\widetilde
O(n^{1/3})}$~\citep{Klivans-Servedio-2001}.

Our reduction framework (Section~\ref{section:base-hardness}) in fact shows a general result that any online sleeping combinatorial optimization problem with two simple structural
properties is as hard as PAC learning DNF expressions.
Leveraging this result, we obtain hardness results for the sleeping variant of well-studied online combinatorial optimization problems for which a polynomial-time no-regret algorithm exists: \prob{Online Minimum Spanning Tree}, \prob{Online
$k$-Subsets}, \prob{Online $k$-Truncated Permutations}, \prob{Online Minimum Cut}, and \prob{Online
Bipartite Matching} (Section~\ref{section:hardness-results}).

Our hardness result applies to the worst-case adversary as well as a \emph{stochastic} adversary, who draws an i.i.d. sample every round from a fixed (but unknown to the learner) joint distribution over availabilities and losses.
This implies that no-regret algorithms would require even stronger restrictions on the adversary. 

\subsection{Related Work}

\paragraph{Online Combinatorial Optimization.} The standard problem of online linear optimization with $d$ actions
(Experts setting) admits algorithms with $O(d)$ running time
per round and $O(\sqrt{T \log d})$ regret after $T$
rounds~\citep{Littlestone-Warmuth-1994, Freund-Schapire-1997}, which is minimax optimal~\citep[Chapter~2]{Cesa-Bianchi-Lugosi-2006}. A naive application of such algorithms to online combinatorial optimization problem (precise definitions to be given momentarily) over a ground set of $d$ elements will result in $\exp(O(d))$ running time per round and $O(\sqrt{Td})$ regret.

Despite this, many online combinatorial optimization problems, such as the ones considered in this paper, admit algorithms with\footnote{In this paper, we use the $\poly(\cdot)$ notation to indicate a polynomially bounded function of the arguments.} $\poly(d)$ running time per round and $O(\poly(d)\sqrt{T})$ regret~\citep{Takimoto-Warmuth-2003, Kalai-Vempala-2005, Koolen-Warmuth-Kivinen-2010, Audibert-Bubeck-Lugosi-2013}. In fact, \citet{Kalai-Vempala-2005} shows that the existence of a polynomial-time algorithm for an offline combinatorial problem implies the existence of an algorithm for the corresponding online optimization problem with the same per-round running time and $O(\poly(d) \sqrt{T})$ regret.

\paragraph{Online Sleeping Optimization.} In studying online sleeping optimization, three different notions of regret
have been used: \begin{inparaenum}[(a)] \item policy regret, \item ranking
regret, and \item per-action regret, \end{inparaenum} in decreasing order of computational hardness to achieve no-regret.
\emph{Policy regret} is the total difference between the loss of the algorithm and the loss of the best policy, which maps a set of available actions and the observed loss sequence to an available action~\citep{Neu-Valko-2014}.
\emph{Ranking regret} is the total difference between
the loss of the algorithm and the loss of the best ranking of actions, which corresponds to a policy that chooses in each round the highest-ranked available
action~ \citep{Kleinberg-Niculescu-Mizil-Sharma-2010, Kanade-Steinke-2014,
Kanade-McMahan-Bryan-2009}. 
\emph{Per-action regret} is the difference between the
loss of the algorithm and the loss of an action, summed over only the rounds in
which the action is available~\citep{Freund-Schapire-Singer-Warmuth-1997,
Koolen-Warmuth-Adamskiy-2015}. 
Note that policy regret upper bounds ranking regret, and while ranking regret
and per-action regret are generally incomparable, per-action regret is
usually the smallest of the three notions.  

The sleeping Experts (also known as Specialists) setting has been extensively studied in the literature \citep{Freund-Schapire-Singer-Warmuth-1997,Kanade-Steinke-2014}. In this paper we focus on the more general online sleeping combinatorial optimization problem, and in particular, the per-action notion of regret.

A summary of known results for online sleeping optimization problems is given in Figure~\ref{fig:summary}. Note in particular that an efficient algorithm was known for minimizing per-action regret in the sleeping Experts problem \citep{Freund-Schapire-Singer-Warmuth-1997}. We show in this paper that a similar efficient algorithm for minimizing per-action regret in online sleeping combinatorial optimization problems cannot exist, unless there is an efficient algorithm for learning DNFs. Our reduction technique is closely related to that of \citet{Kanade-Steinke-2014}, who reduced agnostic learning of disjunctions to \emph{ranking regret} minimization in the sleeping Experts setting.

\begin{figure}[ht] 
\begin{tabular}{l||c|p{13em}|p{13em}} 
	{\bf Regret notion} & {\bf Bound} & {\bf Sleeping Experts} & {\bf Sleeping Combinatorial Opt.} \\
  \hline \hline
\multirow{2}{*}{Policy} & Upper & $O(\sqrt{T \log d})$, under ILA
\citep{Kanade-McMahan-Bryan-2009} & $O(\poly(d)\sqrt{T})$, under ILA  \citep{Neu-Valko-2014,Abbasi-Yadkori-Bartlett-Kanade-Seldin-Szepesvari-2013} \\
\cline{2-4}
& Lower & & $\Omega(\poly(d)T^{1-\delta})$, under SLA \citep{Abbasi-Yadkori-Bartlett-Kanade-Seldin-Szepesvari-2013} \\
\hline
\hline
\multirow{2}{*}{Ranking} & Lower & $\Omega(\poly(d)T^{1-\delta})$, under SLA & $\Omega(\exp(\Omega(d))\sqrt{T})$, under SLA\\
& & \citep{Kanade-Steinke-2014} & [Easy construction, omitted]\\
\hline
\hline
\multirow{2}{*}{Per-action} & Upper & $O(\sqrt{T \log d})$, adversarial setting \\
& & \citep{Freund-Schapire-Singer-Warmuth-1997} \\
\cline{2-4}
& Lower & & $\Omega(\poly(d)T^{1-\delta})$, under SLA \\
& & & [This paper]\\
\hline
\hline
\end{tabular} 
\caption{\label{fig:summary} \small Summary of known results. {\em Stochastic Losses and Availabilities} (SLA) assumption is where adversary chooses a joint distribution over loss and availability before the first round, and takes an i.i.d. sample every round. {\em Independent Losses and Availabilities} (ILA) assumption is where adversary chooses losses and availabilities independently of each other (one of the two may be adversarially chosen; the other one is then chosen i.i.d in each round). Policy regret upper bounds ranking regret which in turn upper bounds per-action regret for the problems of interest; hence some bounds shown in some cells of the table carry over to other cells by implication and are not shown for clarity. The lower bound on ranking regret in online sleeping combinatorial optimization is unconditional and holds for any algorithm, efficient or not. All other lower bounds are {\em computational}, i.e. for polynomial time algorithms, assuming intractability of certain well-studied learning problems, such as learning DNFs or learning noisy parities.}
\end{figure}

\section{Preliminaries}
\label{section:preliminaries}

An instance of online combinatorial optimization is defined by a \emph{ground
set} $U$ of $d$ elements, and a \emph{decision set} $\D$ of actions, each of
which is a subset of $U$.  In each round $t$, the online learner is required to
choose an action $V_t \in \D$, while simultaneously an adversary chooses a loss
function $\ell_t: U \rightarrow [-1, 1]$. The loss of any $V \in \D$ is given
by (with some abuse of notation)
\[\textstyle \ell_t(V)\ :=\ \sum_{e \in V} \ell_t(e).\]
The learner suffers loss $\ell_t(V_t)$ and obtains $\ell_t$ as
feedback. The regret of the learner with respect to an action $V \in \D$ is defined
to be
\[\textstyle \Regret_T(V)\ :=\ \sum_{t=1}^T \ell_t(V_t) - \ell_t(V).\]

We say that an online optimization algorithm has a regret bound of $f(d, T)$ if
$\Regret_T(V) \leq f(d, T)$ for all $V \in \D$. We say that the algorithm has \emph{no regret} if $f(d, T) = \poly(d) T^{1 - \delta}$ for some
$\delta \in (0, 1)$, and it is \emph{computationally efficient} if it has a
per-round running time of order $\poly(d, T)$.

We now define an instance of the \oscombo. In this setting, at the start of each round $t$, the adversary
selects a set of \emph{sleeping elements} $S_t \subseteq U$ and reveals it to
the learner.  Define $\A_t = \{V \in \D ~|~ V \cap S_t = \emptyset\}$, the set
of \emph{awake actions} at round $t$; the remaining actions in $\D$, called
\emph{sleeping actions}, are unavailable to the learner for that round. If
$\A_t$ is empty, i.e., there are no awake actions, then the learner is not
required to do anything for that round and the round is discarded from
computation of the regret.

For the rest of the paper, unless noted otherwise, we use \emph{per-action regret} as our performance
measure. Per-action regret with respect to $V \in \D$ is defined as:
\begin{equation}
\Regret_T(V)\ :=\ \sum_{t:\, V \in \A_t} \ell_t(V_t) - \ell_t(V).
\end{equation}
In other words, our notion of regret considers only the rounds in which $V$ is
awake.  

For clarity, we define an \ocombo \ \emph{problem} as a
family of \emph{instances} of \ocombo \ (and
correspondingly for \oscombo). For example,
\prob{Online Shortest Path} \emph{problem} is the family of all \emph{instances} of all graphs
with designated source and sink vertices, where the decision set $\D$ is a set
of paths from the source to sink, and the elements are edges of the graph. 

Our main result is that many natural online sleeping combinatorial optimization
problems are unlikely to admit a computationally efficient no-regret algorithm,
although their non-sleeping versions (i.e., $\A_t = \D$ for all $t$) do. More
precisely, we show that these online sleeping combinatorial optimization
problems are at least as hard as PAC learning DNF expressions, a long-standing
open problem.

\section{Online Agnostic Learning of Disjunctions}
\label{section:learning-disjunctions}

Instead of directly reducing PAC learning DNF expressions to no-regret learning for online sleeping combinatorial optimization problems, we use an intermediate problem, online agnostic learning of disjunctions. By a standard online-to-batch conversion argument \citep{Kanade-Steinke-2014}, online agnostic learning of disjunctions is at least as hard as agnostic improper PAC-learning of disjunctions~\citep{Kearns-Schapire-Sellie-1994}, which in turn is at least as hard as PAC-learning of DNF expressions \citep{Kalai-Kanade-Mansour-2012}.
The online-to-batch conversion argument allows us to assume the stochastic adversary (i.i.d. input sequence) for online agnostic learning of disjunctions, which in turn implies that our reduction applies to online sleeping combinatorial optimization with a stochastic adversary.

Online agnostic learning of disjunctions is a repeated game between the adversary and a
learning algorithm. Let $n$ denote the number of variables in the disjunction. In each round $t$, the adversary chooses a vector $\x_t \in
\{0,1\}^n$, the algorithm predicts a label $\widehat y_t \in \{0,1\}$ and then
the adversary reveals the correct label $y_t \in \{0,1\}$. If $\widehat y_t \neq
y_t$, we say that algorithm makes an error.

For any predictor $\phi:\{0,1\}^n \to \{0,1\}$,
we define the \emph{regret} with respect to $\phi$ after $T$ rounds as
\[\textstyle \Regret_T(\phi) = \sum_{t=1}^T \indicator[\widehat y_t \neq y_t] - \indicator[\phi(\x_t) \neq y_t].
\]
Our goal is to design an algorithm that is
competitive with any disjunction, i.e. for any disjunction $\phi$ over $n$
variables, the regret is bounded by $\poly(n)\cdot T^{1-\delta}$ for some
$\delta \in (0, 1)$. Recall that a disjunction over $n$ variables is a boolean
function $\phi:\{0,1\}^n \to \{0,1\}$ that on an input $\x = (x(1), x(2),
\dots, x(n))$ outputs
$$
\phi(\x) = \left(\bigvee_{i \in P} x(i) \right) \vee \left(\bigvee_{i \in N} \overline{x(i)} \right)
$$
where $P$ and $N$ are disjoint subsets of $\{1,2,\dots,n\}$. We allow either $P$ or $N$ to be empty, and the empty disjunction is interpreted as the constant $0$ function. For any index $i \in \{1, 2, \ldots, n\}$, we call it a \emph{relevant index} for $\phi$ if $i \in P \cup N$ and \emph{irrelevant index} for $\phi$ otherwise. For any relevant index $i$, we call it \emph{positive} if $i \in P$ and \emph{negative} if $i \in N$.

\section{General Hardness Result}
\label{section:base-hardness}

In this section, we identify two combinatorial properties of \oscombo \ problems that are computationally hard.

\begin{definition}
	Let $n$ be a positive integer. Consider an instance of \oscombo \ where the ground set $U$ has $d$ elements with $3n + 2 \leq d \leq \poly(n)$. This instance is called a \textbf{hard instance with parameter $n$}, if there exists a subset $\Us \subseteq U$ of size $3n + 2$ and a bijection between $\Us$ and the set (i.e., labeling of elements in $\Us$ by the set)
\[ \bigcup_{i=1}^n\{(i, 0), (i, 1), (i, \star)\} \cup \{\F, \T\},\]
such that the decision set $\D$ satisfies the following properties:
\begin{enumerate}
\item {\bf (Heaviness)} Any action $V \in \D$ has at least $n+1$ elements in $\Us$.
\item {\bf (Richness)} For all $(s_1, \ldots, s_{n+1}) \in \{0, 1, \star\}^n \times \{\F, \T\}$, the action $\{(1, s_1), (2, s_2), \ldots, (n, s_n), s_{n+1}\} \in \Us$ is in $\D$.
\end{enumerate}
\end{definition}



We now show how to use the above definition of hard \emph{instances} to prove the hardness of an \oscombo (OSCO) \emph{problem} by reducing from the online agnostic learning of disjunction (OALD) problem. At a high level, the reduction works as follows. Given an instance of the OALD problem, we construct a specific instance of the the OSCO and a sequence of losses and availabilities based on the input to the OALD problem. This reduction has the property that for any disjunction, there is a special set of actions of size $n+1$ such that (a) exactly one action is available in any round and (b) the loss of this action exactly equals the loss of the disjunction on the current input example. Furthermore, the action chosen by the OSCO can be converted into a prediction in the OALD problem with only lesser or equal loss. These two facts imply that the regret of the OALD algorithm is at most $n+1$ times the per-action regret of the OSCO algorithm.



\begin{algorithm}[t]
\caption{\textsc{Algorithm $\algdisj$ for learning disjunctions}
\label{algorithm:generic-disjunction-learner}}
\begin{algorithmic}[1]
\REQUIRE An algorithm $\algosco$ for the online sleeping combinatorial optimization problem, and the input size $n$ for the disjunction learning problem.
\STATE Construct a hard instance $(U, \D)$ with parameter $n$ of the online sleeping combinatorial optimization problem, and run $\algosco$ on it.
\FOR{$t=1,2,\dots,T$}
\STATE Receive $\x_t \in \{0,1\}^n$.
\STATE Set the set of sleeping elements for $\algosco$ to be $S_t = \{(i, 1-x_t(i))\ |\ i = 1, 2, \ldots, n\}$.
\STATE Obtain an action $V_t \in \D$ by running $\algosco$ such that $V_t \cap S_t = \emptyset$.
\STATE Set $\widehat y_t = \indicator[0 \notin V_t]$.
\STATE Predict $\widehat y_t$, and receive true label $y_t$.
\STATE In algorithm $\algosco$, set the loss of the awake elements $e \in U \setminus S_t$ as follows:
	\[
	\ell_t(e) = \begin{cases}
		\frac{1-y_t}{n+1} & \text{ if } e \neq 0 \\
		y_t - \frac{n(1-y_t)}{n+1} & \text{ if } e = 0.
	\end{cases}
	\]
\ENDFOR
\end{algorithmic}
\end{algorithm}

\begin{theorem} \label{thm:main}
Consider an online sleeping combinatorial optimization problem such that for any positive integer $n$, there is a hard instance with parameter $n$ of the problem. Suppose there is an algorithm $\algosco$ that for any instance of the problem with ground set $U$ of size $d$, runs in time $\poly(T, d)$ and has regret bounded by $\poly(d) \cdot T^{1-\delta}$ for some $\delta \in (0, 1)$. Then, there exists an algorithm $\algdisj$ for online agnostic learning of disjunctions over $n$ variables with running time $\poly(T,n)$ and regret $\poly(n) \cdot T^{1-\delta}$.
\end{theorem}

\begin{proof}
$\algdisj$ is given in Algorithm~\ref{algorithm:generic-disjunction-learner}. First, we note that in each round $t$, we have
\begin{equation} \label{eq:loss-algorithm}
\ell_t(V_t)\ \geq\ \indicator[y_t \neq \widehat{y}_t].
\end{equation}
We prove this separately for two different cases; in both cases, the inequality follows from the heaviness property, i.e., the fact that $|V_t| \geq n + 1$.
\begin{enumerate}
\item If $\F \notin V_t$, then the prediction of $\algdisj$ is $\widehat{y}_t = 1$, and thus
\[ \ell_t(V_t) = |V_t| \cdot \frac{1 - y_t}{n+1}\
\geq\ 1 - y_t\ =\ \indicator[y_t \neq \widehat{y}_t]. \]

\item If $\F \in V_t$, then the prediction of $\algdisj$ is $\widehat{y}_t = 0$, and thus
\[\ell_t(V_t)\ = \ (|V_t| - 1) \cdot \frac{1-y_t}{n+1} +  \left(y_t - \frac{n(1-y_t)}{n+1}\right) \ \geq\ y_t\ =\ \indicator[y_t \neq \widehat{y}_t].\]
\end{enumerate}
Note that if $V_t$ satisfies the equality $|V_t| = n+1$, then we have an equality $\ell_t(V_t) =
\indicator[y_t \neq \widehat{y}_t]$; this property will be useful later.

Next, let $\phi$ be an arbitrary disjunction, and let $i_1 < i_2 < \dots < i_m$ be its relevant indices sorted in increasing order.
Define $f_\phi:
\{1, 2, \ldots, m\} \rightarrow \{0, 1\}$ as $f_\phi(j) := \indicator[i_j \text{ is a positive index for } \phi]$, and define the set of elements $W_\phi := \{(i, \star)\ |\ i \text{ is an irrelevant index for } \phi\}$.
Finally, let $\D_{\phi} = \{ V_\phi^1, V_\phi^2, \dots, V_\phi^{m+1} \}$ be the set of $m + 1$ actions where for $j = 1, 2, \dots, m$, we define
\[ V_\phi^j\ :=\ \{(i_\ell, 1 - f_\phi(\ell))\ |\ 1 \leq \ell < j\} \cup \{(i_j, f_\phi(j))\} \cup \{(i_\ell, \star)\ |\ j < \ell \leq m\} \cup W_\phi \cup \{\T\}, \]
and
\[V_\phi^{m+1}\ :=\ \{(i_\ell, 1 - f_\phi(\ell))\ |\ 1 \leq \ell \leq m\} \cup W_\phi \cup \{\F\}.\]
The actions in $\D_{\phi}$ are indeed in the decision set $\D$ due to the richness property.

We claim that $\D_\phi$ contains exactly one awake action in every round and the awake action contains the element $\T$ if and only if $\phi(\x_t) = 1$. First, we prove uniqueness: if $V_{\phi}^{j}$ and $V_{\phi}^{k}$ (where $j < k$) are both awake in the same round, then $(i_j, f_\phi(j)) \in V_{\phi}^{j}$ and $(i_j, 1 - f_{\phi}(j)) \in V_{\phi}^{k}$ are both awake elements, contradicting our choice of $S_t$.
To prove the rest of the claim, we consider two cases:
\begin{enumerate}
\item If $\phi(\x_t) = 1$, then there is at least one $j \in \{1, 2, \dots,
m\}$ such that $x_t(i_j) = f_\phi(j)$. Let $j'$ be the smallest such
$j$. Then, by construction, the set $V_\phi^{j'}$ is awake at time $t$, and $\T
\in V_\phi^{j'}$, as required.

\item If $\phi(\x_t) = 0$, then for all $j \in \{1, 2, \dots, m\}$ we must
have $x_t(i_j) = 1 - f_\phi(j)$. Then, by construction, the set $V_\phi^{m+1}$
is awake at time $t$, and $\F \in V_\phi^{m+1}$, as required.
\end{enumerate}

Since every action in $\D_\phi$ has exactly $n + 1$ elements, and if $V$ is awake action in $\D_\phi$ at time $t$, we just showed that $\T \in V$ if
and only if $\phi(\x_t) = 1$, exactly the same argument as in the beginning of
this proof implies that
\begin{equation} \label{eq:loss-disjunction}
\ell_t(V)\ =\ \indicator[y_t \neq \phi(\x_t)].
\end{equation}
Furthermore, since exactly one action in $\D_\phi$ is awake every round, we have
\begin{equation} \label{eq:total-loss-disjunction}
\sum_{t=1}^T \indicator[y_t \neq \phi(\x_t)]\ =\ \sum_{V \in \D_\phi}\ \sum_{t:\, V \in \A_t} \ell_t(V).
\end{equation}

Finally, we can bound the regret of algorithm $\algdisj$ (denoted $\Regret_T^\textsf{disj}$) in
terms of the regret of algorithm $\algosco$ (denoted $\Regret_T^\textsf{osco}$) as follows:
\begin{align*}
\Regret_T^\textsf{disj}(\phi)
& = \sum_{t=1}^T \indicator[\widehat y_t \neq y_t] - \indicator[\phi(\x_t) \neq y_t] \le \sum_{V \in \D_\phi}\ \sum_{t:\, V \in \A_t} \ell_t(V_t) - \ell_t(V) \\
&= \sum_{V \in \D_\phi} \Regret_T^\textsf{osco}(V) \le |\D_\phi| \cdot \poly(d) \cdot T^{1 - \delta}= \poly(n) \cdot T^{1 - \delta},
\end{align*}
The first inequality follows by \eqref{eq:loss-algorithm} and \eqref{eq:total-loss-disjunction}, and the last equation since $|\D_{\phi}| \leq n+1$ and $d \leq \poly(n)$.
\end{proof}




\subsection{Hardness results for Policy Regret and Ranking Regret}
It is easy to see that our technique for proving hardness easily extends to ranking regret (and therefore, policy regret). The reduction simply uses any algorithm for minimizing ranking regret in Algorithm~\ref{algorithm:generic-disjunction-learner} as $\algosco$. This is because in the proof of Theorem~\ref{thm:main}, the set $\D_\phi$ has the property that exactly one action $V_t \in \D_\phi$ is awake in any round $t$, and $\ell_t(V_t) = \indicator[y_t \neq \widehat{y}_t]$. Thus, if we consider a ranking where the actions in $\D_\phi$ are ranked at the top positions (in arbitrary order), the loss of this ranking exactly equals the number of errors made by the disjunction $\phi$ on the input sequence. The same arguments as in the proof of Theorem~\ref{thm:main} then imply that the regret of $\algdisj$ is bounded by that of $\algosco$, implying the hardness result.


\section{Hard Instances for Specific Problems}
\label{section:hardness-results}

Now we apply Theorem \ref{thm:main} to prove that many online
sleeping combinatorial optimization problems are as hard as PAC learning DNF expressions by constructing hard instances for them. Note that all these problems admit efficient no-regret algorithms in the
non-sleeping setting.

\subsection{Online Shortest Path Problem}

In the \prob{Online Shortest Path} problem, the learner is given a directed graph $G=
(V,E)$ and designated source and sink vertices $s$ and $t$. The ground set is the set of edges, i.e. $U = E$, and the decision set $\D$ is the set of all paths from $s$ to $t$.  The sleeping version of this problem has been called the \prob{Online Sabotaged Shortest Path} problem by
\cite{Koolen-Warmuth-Adamskiy-2015}, who posed the open question of whether it
admits an efficient no-regret algorithm. For any $n \in \N$, a hard instance is the graph $G^{(n)}$ shown in Figure~\ref{figure:graph}. It has $3n+2$ edges that are labeled by the elements
of ground set $U = \bigcup_{i=1}^n\{(i, 0), (i, 1), (i, \star)\} \cup
\{\F, \T\}$, as required. Now note that any $s$-$t$ path in this graph has length
exactly $n+1$, so $\D$ satisfies the heaviness property. Furthermore, the
richness property is clearly satisfied, since for any $s \in \{0, 1, \star\}^n
\times \{\F, \T\}$, the set of edges $\{ (1, s_1), (2, s_2), \ldots, (n, s_n),
s_{n+1}\}$ is an $s$-$t$ path and therefore in $\D$. 

\begin{figure}[t]
\centering
\includegraphics{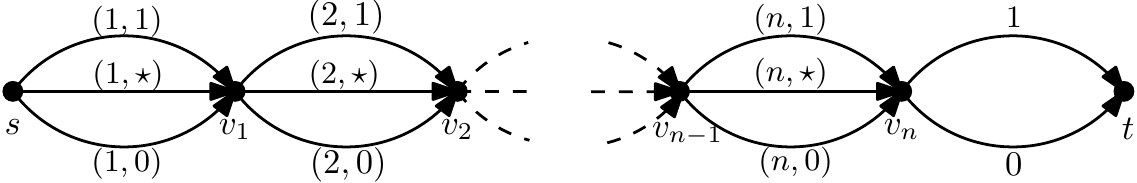}
\caption{\label{figure:graph} Graph $G^{(n)}$.}
\end{figure}

\subsection{Online Minimum Spanning Tree Problem}

In the \prob{Online Minimum Spanning Tree} problem, the learner is given a fixed graph
$G = (V, E)$. The ground set here is the set of edges, i.e. $U = E$, and the
decision set $\D$ is the set of spanning trees in the graph. For any $n \in \N$, a hard instance is the same graph $G^{(n)}$ shown in Figure~\ref{figure:graph}, except that the edges are undirected. Note that the spanning trees in $G^{(n)}$ are exactly the paths from $s$ to $t$. The hardness of this problem immediately follows from the hardness of the \prob{Online Shortest
Paths} problem.

\subsection{Online \texorpdfstring{$k$}{k}-Subsets Problem}

In the \prob{Online $k$-Subsets} problem, the learner is given a fixed ground set of
elements $U$. The decision set $\D$ is the set of subsets of $U$ of size $k$.
For any $n \in \N$, we construct a hard instance with parameter $n$ of the \prob{Online $k$-Subsets} problem with $k = n + 1$ and $d = 3n+2$. The set $\D$ of all subsets of size $k = n+1$ of a ground set $U$ of size $d = 3n+2$ clearly satisfies both the heaviness and richness properties.

\subsection{Online \texorpdfstring{$k$}{k}-Truncated Permutations Problem}
\begin{figure}[t]
\centering
\includegraphics{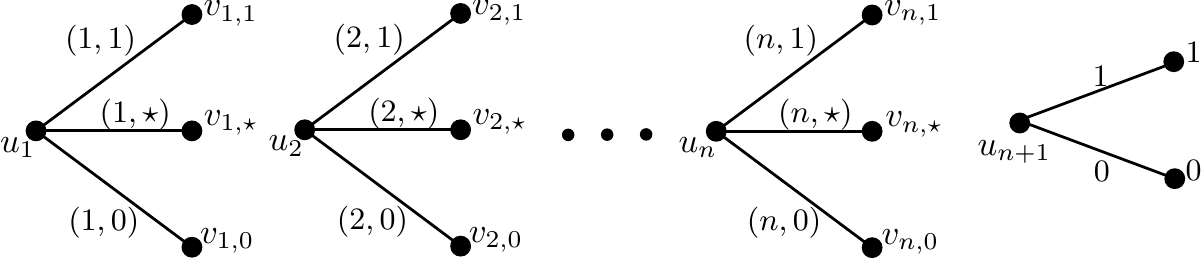}
\caption{\label{figure:permutations} Graph $P^{(n)}$. This is a complete bipartite graph as described in the text, but only the special labeled edges shown for clarity.}
\end{figure}

In the \prob{Online $k$-truncated Permutations} problem (also called the \prob{Online $k$-ranking}
problem), the learner is given a complete bipartite graph with $k$ nodes on one
side and $m \geq k$ nodes on the other, and the ground set $U$ is the set of
all edges; thus $d = km$. The decision set $\D$ is the set of all maximal
matchings, which can be interpreted as truncated permutations of $k$ out of $m$
objects. For any $n \in \N$, we construct a hard instance with parameter $n$ of the \prob{Online
$k$-Truncated Permutations} problem with $k = n + 1$, $m = 3n+2$ and $d = km =
(n+1)(3n+2)$. Let $L = \{u_1, u_2, \ldots, u_{n+1}\}$ be the nodes on the left side of the
bipartite graph, and since $m = 3n+2$, let $R = \{v_{i, 0}, v_{i, 1}, v_{i,
\star}\ |\ i = 1, 2, \ldots, n\} \cup \{v_0, v_1\}$ denote the nodes on the
right side of the graph. The ground set $U$ consists of all $d = km =
(n+1)(3n+2)$ edges joining nodes in $L$ to nodes in $R$. We now specify the
special $3n+2$ elements of the ground set $U$: for $i = 1, 2, \ldots, n$, label
the edges $(u_i, v_{i, 0}), (u_i, v_{i, 1}), (u_i, v_{i, \star})$ by $(i, 0),
(i, 1), (i, \star)$ respectively. Finally, label the edges $(u_{n+1}, v_0),
(u_{n+1}, v_1)$ by $0$ and $1$ respectively. The resulting bipartite graph
$P^{(n)}$ is shown in Figure~\ref{figure:permutations}, where only the special
labeled edges are shown for clarity.

Now note that any maximal matching in this graph has exactly $n+1$ edges, so
the heaviness condition is satisfied. Furthermore, the richness property is
satisfied, since for any $s \in \{0, 1, \star\}^n \times \{\F, \T\}$, the set of
edges $\{ (1, s_1), (2, s_2), \ldots, (n, s_n), s_{n+1}\}$ is a maximal
matching and therefore in $\D$.

\subsection{Online Bipartite Matching Problem}

In the \prob{Online Bipartite Matching} problem, the learner is given a fixed
bipartite graph $G = (V,E)$. The ground set here is the set of edges, i.e. $U =
E$, and the decision set $\D$ is the set of maximal matchings in $G$. For any $n \in \N$, a hard instance with parameter $n$ is the graph $M^{(n)}$ shown in
Figure~\ref{figure:matching}. It has $3n+2$ edges that are labeled by the
elements of ground set $U = \bigcup_{i=1}^n\{(i, 0), (i, 1), (i, \star)\} \cup \{\F, \T\}$, as required. Now note that any maximal matching in this
graph has size exactly $n+1$, so $\D$ satisfies the heaviness property.
Furthermore, the richness property is clearly satisfied, since for any $s \in
\{0, 1, \star\}^n \times \{\F, \T\}$, the set of edges $\{ (1, s_1), (2, s_2),
\ldots, (n, s_n), s_{n+1} \}$ is a maximal matching and therefore in $\D$. 

\begin{figure}[t]
\centering
\includegraphics[width=12cm]{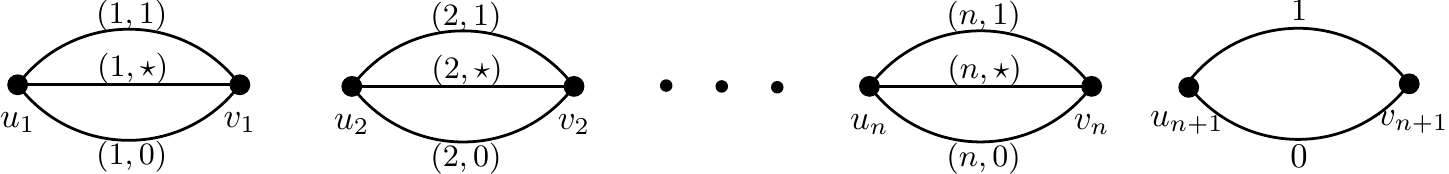}
\caption{\label{figure:matching} Graph $M^{(n)}$ for the \prob{Online Bipartite Matching} problem.}
\end{figure}

\subsection{Online Minimum Cut Problem}

In the \prob{Online Minimum Cut} problem the learner is given a fixed graph
$G = (V,E)$ with a designated pair of vertices $s$ and $t$. The ground set here
is the set of edges, i.e. $U = E$, and the decision set $\D$ is the set of cuts
separating $s$ and $t$: a cut here is a set of edges that when removed from the
graph disconnects $s$ from $t$. For any $n \in \N$, a hard instance is the graph $C^{(n)}$ shown in Figure~\ref{figure:mincut}. It has $3n+2$ edges that are labeled by the
elements of ground set $U = \bigcup_{i=1}^n\{(i, 0), (i, 1), (i, \star)\} \cup \{\F, \T\}$, as required. Now note that any cut in this graph has size
at least $n+1$, so $\D$ satisfies the heaviness property. Furthermore, the
richness property is clearly satisfied, since for any $s \in \{0, 1, \star\}^n
\times \{\F, \T\}$, the set of edges $\{ (1, s_1), (2, s_2), \ldots, (n, s_n),
s_{n+1} \}$ is a cut and therefore in $\D$. 

\begin{figure}[t]
\centering
\includegraphics[width=5cm]{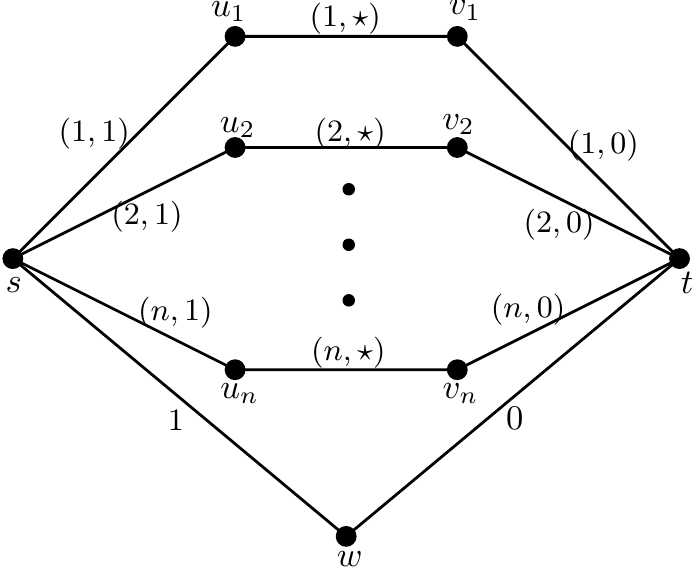}
\caption{\label{figure:mincut} Graph $C^{(n)}$ for the \prob{Online Minimum Cut} problem.}
\end{figure}

\section{Conclusion}

In this paper we showed that obtaining an efficient no-regret algorithm
for sleeping versions of several natural online combinatorial optimization
problems is as hard as efficiently PAC learning DNF expressions, a
long-standing open problem. Our reduction technique requires only very modest
conditions for hard instances of the problem of interest, and in fact is
considerably more flexible than the specific form presented in this paper. We
believe that almost any natural combinatorial optimization problem that
includes instances with exponentially many solutions will be a hard problem in
its online sleeping variant. Furthermore, our hardness result is via stochastic
i.i.d. availabilities and losses, a rather benign form of adversary. This
suggests that obtaining sublinear per-action regret is perhaps a rather hard
objective, and suggests that to obtain efficient algorithms we might need to either (a) make suitable simplifications of the regret criterion or (b) restrict the adversary's power. 

\bibliographystyle{plainnat}
\bibliography{biblio}

\begin{thebibliography}{17}
\providecommand{\natexlab}[1]{#1}
\providecommand{\url}[1]{\texttt{#1}}
\expandafter\ifx\csname urlstyle\endcsname\relax
  \providecommand{\doi}[1]{doi: #1}\else
  \providecommand{\doi}{doi: \begingroup \urlstyle{rm}\Url}\fi

\bibitem[Abbasi-Yadkori et~al.(2013)Abbasi-Yadkori, Bartlett, Kanade, Seldin,
  and Szepesv\'ari]{Abbasi-Yadkori-Bartlett-Kanade-Seldin-Szepesvari-2013}
Yasin Abbasi-Yadkori, Peter~L. Bartlett, Varun Kanade, Yevgeny Seldin, and
  Csaba Szepesv\'ari.
\newblock Online learning in markov decision processes with adversarially
  chosen transition probability distributions.
\newblock In \emph{Advances in Neural Information Processing Systems (NIPS)},
  pages 2508--2516, 2013.

\bibitem[Audibert et~al.(2013)Audibert, S\'ebastien, and
  Lugosi]{Audibert-Bubeck-Lugosi-2013}
Jean-Yves Audibert, Bubeck S\'ebastien, and G\'abor Lugosi.
\newblock Regret in online combinatorial optimization.
\newblock \emph{Mathematics of Operations Research}, 39\penalty0 (1):\penalty0
  31--45, 2013.

\bibitem[Cesa-Bianchi and Lugosi(2006)]{Cesa-Bianchi-Lugosi-2006}
Nicol\`o Cesa-Bianchi and G\'abor Lugosi.
\newblock \emph{Prediction, Learning and Games}.
\newblock Cambridge University Press, New York, NY, 2006.

\bibitem[Freund and Schapire(1997)]{Freund-Schapire-1997}
Yoav Freund and Robert~E. Schapire.
\newblock A decision-theoretic generalization of on-line learning and an
  application to boosting.
\newblock \emph{Journal of Computer and System Sciences}, 55\penalty0
  (1):\penalty0 119--139, 1997.

\bibitem[Freund et~al.(1997)Freund, Schapire, Singer, and
  Manfred]{Freund-Schapire-Singer-Warmuth-1997}
Yoav Freund, Robert~E. Schapire, Yoram Singer, and Warmuth~K. Manfred.
\newblock Using and combining predictors that specialize.
\newblock In \emph{Proceedings of the 29th Annual ACM symposium on Theory of
  Computing}, pages 334--343. ACM, 1997.

\bibitem[Kalai and Vempala(2005)]{Kalai-Vempala-2005}
Adam Kalai and Santosh Vempala.
\newblock Efficient algorithms for online decision problems.
\newblock \emph{Journal of Computer and System Sciences}, 71\penalty0
  (3):\penalty0 291--307, 2005.

\bibitem[Kalai et~al.(2012)Kalai, Kanade, and
  Mansour]{Kalai-Kanade-Mansour-2012}
Adam~Tauman Kalai, Varun Kanade, and Yishay Mansour.
\newblock Reliable agnostic learning.
\newblock \emph{Journal of Computer and System Sciences}, 78\penalty0
  (5):\penalty0 1481--1495, 2012.

\bibitem[Kanade and Steinke(2014)]{Kanade-Steinke-2014}
Varun Kanade and Thomas Steinke.
\newblock Learning hurdles for sleeping experts.
\newblock \emph{ACM Transactions on Computation Theory (TOCT)}, 6\penalty0
  (3):\penalty0 11, 2014.

\bibitem[Kanade et~al.(2009)Kanade, McMahan, and
  Bryan]{Kanade-McMahan-Bryan-2009}
Varun Kanade, H.~Brendan McMahan, and Brent Bryan.
\newblock Sleeping experts and bandits with stochastic action availability and
  adversarial rewards.
\newblock In \emph{Proceedings of the 12th International Conference on
  Artificial Intelligence and Statistics (AISTATS)}, pages 272--279, 2009.

\bibitem[Kearns et~al.(1994)Kearns, Schapire, and
  Sellie]{Kearns-Schapire-Sellie-1994}
Michael~J. Kearns, Robert~E. Schapire, and Linda~M. Sellie.
\newblock Toward efficient agnostic learning.
\newblock \emph{Machine Learning}, 17\penalty0 (2--3):\penalty0 115--141, 1994.

\bibitem[Kleinberg et~al.(2010)Kleinberg, Niculescu-Mizil, and
  Sharma]{Kleinberg-Niculescu-Mizil-Sharma-2010}
Robert Kleinberg, Alexandru Niculescu-Mizil, and Yogeshwer Sharma.
\newblock Regret bounds for sleeping experts and bandits.
\newblock \emph{Machine learning}, 80\penalty0 (2-3):\penalty0 245--272, 2010.

\bibitem[Klivans and Servedio(2001)]{Klivans-Servedio-2001}
Adam~R. Klivans and Rocco Servedio.
\newblock {Learning {DNF} in Time $2^{\tilde {O}(n^{1/3})}$}.
\newblock In \emph{Proceedings of the 33rd Annual ACM Symposium on Theory of
  Computing (STOC)}, pages 258--265. ACM, 2001.

\bibitem[Koolen et~al.(2010)Koolen, Warmuth, and
  Kivinen]{Koolen-Warmuth-Kivinen-2010}
Wouter~M. Koolen, Manfred~K. Warmuth, and Jyrki Kivinen.
\newblock Hedging structured concepts.
\newblock In Adam~Tauman Kalai and Mehryar Mohri, editors, \emph{Proceedings of
  the 23th Conference on Learning Theory (COLT)}, pages 93--105, 2010.

\bibitem[Koolen et~al.(2015)Koolen, Warmuth, and
  Adamskiy]{Koolen-Warmuth-Adamskiy-2015}
Wouter~M. Koolen, Manfred~K. Warmuth, and Dmitry Adamskiy.
\newblock Open problem: Online sabotaged shortest path.
\newblock In \emph{Proceedings of the 28th Conference on Learning Theory
  (COLT)}, 2015.

\bibitem[Littlestone and Warmuth(1994)]{Littlestone-Warmuth-1994}
Nick Littlestone and Manfred~K. Warmuth.
\newblock The weighted majority algorithm.
\newblock \emph{Information and computation}, 108\penalty0 (2):\penalty0
  212--261, 1994.

\bibitem[Neu and Valko(2014)]{Neu-Valko-2014}
Gergely Neu and Michal Valko.
\newblock Online combinatorial optimization with stochastic decision sets and
  adversarial losses.
\newblock In \emph{Advances in Neural Information Processing Systems}, pages
  2780--2788, 2014.

\bibitem[Takimoto and Warmuth(2003)]{Takimoto-Warmuth-2003}
Eiji Takimoto and Manfred~K. Warmuth.
\newblock Path kernels and multiplicative updates.
\newblock \emph{The Journal of Machine Learning Research}, 4:\penalty0
  773--818, 2003.

\end{thebibliography}

\end{document}